\documentclass{article}

\usepackage{times}
\usepackage{graphicx} % more modern
\usepackage{subfigure}

\usepackage{natbib}

\usepackage{algorithm}
\usepackage{algorithmic}

\usepackage{hyperref}

\usepackage[accepted]{icml2017}

\usepackage[utf8]{inputenc} % allow utf-8 input
\usepackage[T1]{fontenc}    % use 8-bit T1 fonts
\usepackage{hyperref}       % hyperlinks
\usepackage{url}            % simple URL typesetting
\usepackage{booktabs}       % professional-quality tables
\usepackage{amsfonts}       % blackboard math symbols
\usepackage{nicefrac}       % compact symbols for 1/2, etc.
\usepackage{microtype}      % microtypography

\usepackage{graphicx}
\usepackage{caption}
\usepackage{url}
\usepackage{subfigure}
\usepackage{amsmath}
\usepackage{bbm}
\usepackage{amssymb}
\usepackage{colortbl}
\usepackage{booktabs}
\usepackage{stmaryrd}
\usepackage{float}
\usepackage{dsfont} % pour la fonction indicatrice 1 double barre: \mathds{1}
\usepackage{upgreek} % d'autres lettres grecques
\usepackage[raggedright]{titlesec}
\usepackage{tabularx}
\usepackage{lipsum}
\renewcommand{\bibfont}{\small}
\usepackage{adjustbox}
\usepackage{wrapfig}
\usepackage{enumitem}

\usepackage{xcolor}
\definecolor{darkgreen}{rgb}{0,0.6,0}

\renewcommand{\cite}[1]{\citep{#1}}

\newcounter{lastromanpg}

\def\X{{\mathcal{X}}}

\def\L{{\mathcal{L}}}
\def\R{{\mathbb{R}}}
\def\Id{{\mathbb{I}}}

\usepackage{amsthm}
\newtheorem{theorem}{Theorem}

\newtheorem{defin}{Definition}

\newcommand{\vertiii}[1]{{\left\vert\kern-0.25ex\left\vert\kern-0.25ex\left\vert #1
    \right\vert\kern-0.25ex\right\vert\kern-0.25ex\right\vert}}

\icmltitlerunning{Sharp Minima Can Generalize For Deep Nets}

\begin{document}

\twocolumn[
\icmltitle{Sharp Minima Can Generalize For Deep Nets}



% you can specify symbols, otherwise they are numbered in order
\icmlsetsymbol{equal}{*}

\begin{icmlauthorlist}
\icmlauthor{Laurent Dinh}{udem}
\icmlauthor{Razvan Pascanu}{deepmind}
\icmlauthor{Samy Bengio}{google}
\icmlauthor{Yoshua Bengio}{udem,cifar}
\end{icmlauthorlist}

\icmlaffiliation{udem}{Universit\'e of Montr\'eal, Montr\'eal, Canada}
\icmlaffiliation{google}{Google Brain, Mountain View, United States}
\icmlaffiliation{deepmind}{DeepMind, London, United Kingdom}
\icmlaffiliation{cifar}{CIFAR Senior Fellow}

\icmlcorrespondingauthor{Laurent Dinh}{laurent.dinh@umontreal.ca}

\icmlkeywords{geometry, optimization, generalization, machine learning, ICML}

\vskip 0.3in
]



\printAffiliationsAndNotice{}  % leave blank if no need to mention equal contribution
%\footnotetext{hi}

\begin{abstract}
    Despite their overwhelming capacity to overfit, deep learning architectures
    tend to generalize relatively well to unseen data, allowing them to be
    deployed in practice. However, explaining why this is the case is still an
    open area of research. One standing hypothesis that is gaining popularity,
    e.g. \citet{hochreiter1997flat, keskar2016large}, is that the flatness of
    minima of the loss function found by stochastic gradient based methods
    results in good generalization. This paper argues that most notions of
    flatness are problematic for deep models and can not be directly applied to
    explain generalization.  Specifically, when focusing on deep networks with
    rectifier units, we can exploit the particular geometry of parameter space
    induced by the inherent symmetries that these architectures exhibit to build
    equivalent models corresponding to arbitrarily sharper minima.
    Furthermore, if we allow to reparametrize
    a function, the geometry of its parameters can change drastically without
    affecting its generalization properties.
\end{abstract}

\section{Introduction}

{Deep learning} techniques have been very successful in several domains, like
\emph{object recognition in images}~\citep[e.g][]{krizhevsky2012imagenet,
simonyan2014very, szegedy2015going, he2016deep}, \emph{machine
translation}~\citep[e.g.][]{cho2014learning, sutskever2014sequence,
Bahdanau-et-al-ICLR2015-small, wu2016google, gehring2016convolutional} and \emph{speech
recognition}~\citep[e.g.][]{graves2013speech, HannunCCCDEPSSCN14,
chorowski2015attention, chan2015listen, collobert2016wav2letter}.
Several arguments have been brought forward to justify these empirical
results. From a representational point of view, it has been argued that
deep networks can efficiently approximate certain functions~\citep[e.g.][]{
montufar2014number, raghu2016expressive}.
Other works \citep[e.g][]{DauphinSaddle14, ChoromanskaHMAL15}
have looked at the structure of the
error surface to analyze how trainable these models are. Finally, another point of discussion is how well these models
can generalize~\citep{nesterov2008confidence, keskar2016large, zhang2016understanding}.
These correspond, respectively, to low \emph{approximation}, \emph{optimization}
and \emph{estimation} error as described by~\citet{bottou2010large}.

Our work focuses on the analysis of the estimation error. In particular,
different approaches had been used to look at the question of why
\emph{stochastic gradient descent} results in
solutions that generalize well~\citep{bottou-lecun-2004a, bottou-bousquet-2008}. For
example, \citet{duchi2011adaptive,
nesterov2008confidence, hardt2015train, bottou2016optimization, gonen2017fast}
rely on the concept of \emph{stochastic
approximation} or \emph{uniform stability}~\citep{bousquet2002stability}. Another conjecture that was recently~\citep{keskar2016large}
explored, but that could be traced back to~\citet{hochreiter1997flat},
relies on the geometry of the
loss function around a given solution. It argues that flat minima, for some
definition of flatness, lead to better
generalization.  Our work focuses on this particular conjecture, arguing that
there are critical issues when applying the concept of flat minima to deep neural
networks, which require rethinking what flatness actually means.

While the concept of flat minima is not well defined, having slightly different
meanings in different works, the intuition is relatively simple.  If one
imagines the error as a one-dimensional curve, a minimum is flat if there is a
wide region around it with roughly the same error, otherwise the minimum is
sharp. When moving to higher dimensional spaces, defining flatness becomes more
complicated. In~\citet{hochreiter1997flat} it is defined as the size of the
connected region around the minimum where the training loss is relatively
similar. \citet{chaudhari2016entropy} relies, in contrast, on the curvature of
the second order structure around the minimum, while \citet{keskar2016large}
looks at the maximum loss in a bounded neighbourhood of the minimum.  All these
works rely on the fact that flatness results in robustness to low precision
arithmetic or noise in the parameter space, which, using an
\emph{minimum description length}-based argument, suggests a low expected overfitting.

However, several common architectures and parametrizations in deep learning are
already at odds with this conjecture, requiring at least some degree of
refinement in the statements made. In particular, we show how the
geometry of the associated parameter space can alter the ranking between
prediction functions when considering several measures of
\emph{flatness/sharpness}. We believe the reason for this contradiction stems from
the Bayesian arguments about KL-divergence made to justify the generalization
ability of flat minima~\citep{hinton1993keeping}.  Indeed, {\em Kullback-Liebler divergence is invariant to change of parameters whereas the notion of "flatness" is not}.
The demonstrations of \citet{hochreiter1997flat} are approximately based on a
Gibbs formalism and rely on strong assumptions and approximations that can
compromise the applicability of the argument, including the assumption of a
discrete function space.

\section{Definitions of flatness/sharpness}
\begin{figure}[h]
\vspace{15pt}
   \centering \includegraphics[width=.45\textwidth]{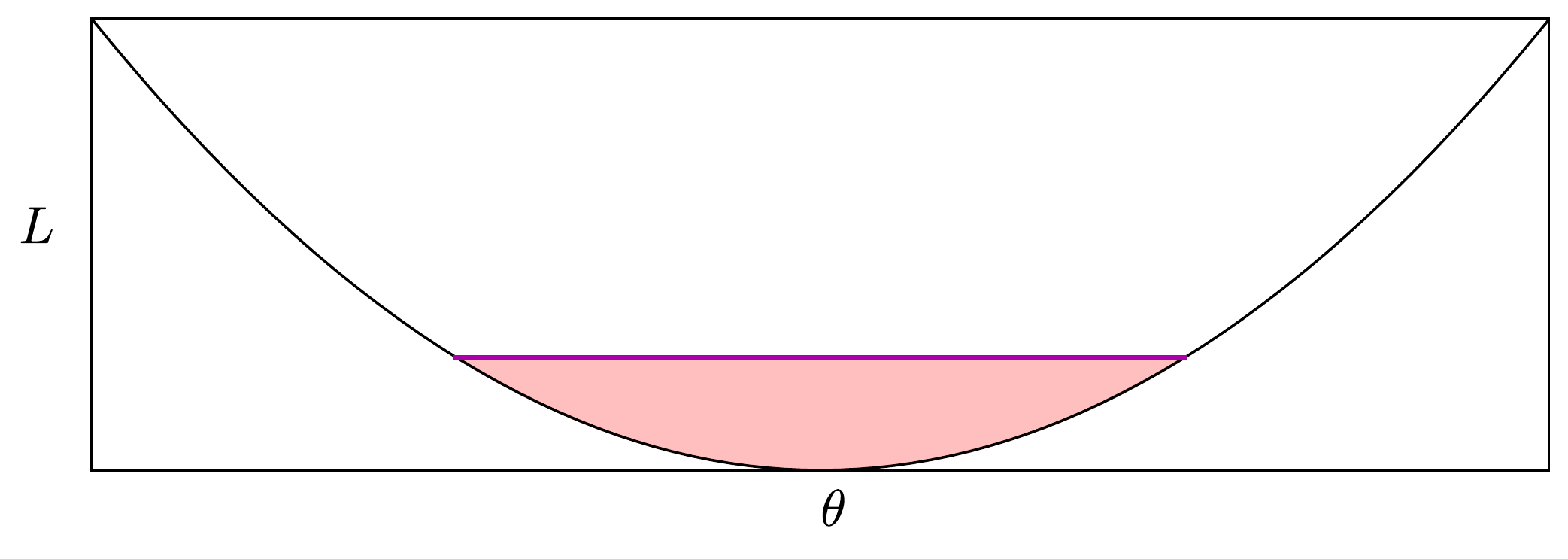}
   \caption{An illustration of the notion of flatness. The loss $L$ as
            a function of $\theta$ is plotted in black. If the height of the
            red area is $\epsilon$, the width will represent the
            volume $\epsilon$-flatness. If the width is $2 \epsilon$, the
            height will then represent the $\epsilon$-sharpness.
            Best seen with colors. }
   \label{fig:flatness}
   \vspace{15pt}
\end{figure}
\label{sec:flatness}

For conciseness, we will restrict ourselves to supervised scalar output
problems, but several conclusions in this paper can apply to other problems as
well. We will consider a function $f$ that takes as input an element $x$ from
an input space $\X$ and outputs a scalar $y$. We will denote by $f_{\theta}$
the prediction function. This prediction function will be parametrized by a
parameter vector $\theta$ in a parameter space $\Theta$. Often, this prediction
function will be over-parametrized and two parameters $(\theta, \theta') \in
\Theta^2$ that yield the same prediction function everywhere, $\forall x \in
\X, f_{\theta}(x) = f_{\theta'}(x)$, are called \emph{observationally
equivalent}. The model is trained to minimize a continuous loss function $L$ which
takes as argument the prediction function $f_{\theta}$.
We will often think of the loss $L$ as a function of $\theta$ and
adopt the notation $L(\theta)$.

The notion of flatness/sharpness of a minimum is relative, therefore we will
discuss metrics that can be used to compare the relative flatness between two
minima. In this section we will formalize three used definitions of flatness
in the literature.

\citet{hochreiter1997flat} defines a flat minimum as "a large connected region
in weight space where the error remains approximately constant". We interpret
this formulation as follows:
\begin{defin}
Given $\epsilon > 0$, a minimum $\theta$, and a loss $L$,
    we define $C(L, \theta, \epsilon)$ as
    the largest (using inclusion as the partial order over the subsets of
    $\Theta$) connected set containing $\theta$ such that $\forall \theta' \in
    C(L, \theta, \epsilon), L(\theta') < L(\theta) + \epsilon$. The
    $\epsilon$-flatness will be defined as the volume of $C(L, \theta,
    \epsilon)$. We will call this measure the volume $\epsilon$-flatness.
\end{defin}
In Figure~\ref{fig:flatness}, $C(L, \theta, \epsilon)$ will be the purple line at
the top of the red area if the height is $\epsilon$ and its volume will simply be
the length of the purple line.

Flatness can also be defined using the local curvature of the loss function
around the minimum if it is a critical point \footnote{In this paper, we will often assume that is the case when dealing with Hessian-based measures in order to have them well-defined.}.
\citet{chaudhari2016entropy, keskar2016large} suggest
that this information is encoded in the eigenvalues of the Hessian.
However, in order to compare how flat one minimum versus another,
the eigenvalues need to be reduced to a single number. Here we consider
the \emph{spectral norm and trace of the Hessian}, two typical measurements
of the eigenvalues of a matrix.

Additionally \citet{keskar2016large} defines the notion of
$\epsilon$-sharpness. In order to make proofs more readable, we will slightly
modify their definition. However, because of norm equivalence in finite
dimensional space, our results will transfer to the original definition
in full space as well. Our modified definition is the following:

\begin{defin}
Let $B_2(\epsilon, \theta)$ be an Euclidean ball centered on a minimum $\theta$
    with radius $\epsilon$. Then, for a non-negative valued loss function $L$, the
    $\epsilon$-sharpness will be defined as proportional to
\begin{align*}
\frac{\max_{\theta' \in B_2(\epsilon, \theta)}\big(L(\theta') - L(\theta)\big)}{1 + L(\theta)}.
\end{align*}
\end{defin}
In Figure~\ref{fig:flatness}, if the width of the red area is $2 \epsilon$ then
the height of the red area is
$\max_{\theta' \in B_2(\epsilon, \theta)}\big(L(\theta') - L(\theta)\big)$.

$\epsilon$-sharpness can be related to the spectral norm of the Hessian. Indeed, a second-order Taylor expansion of $L$ around a critical point minimum is written
\begin{align*}
L(\theta') = ~ &L(\theta) + \frac{1}{2} ~ (\theta' - \theta) ~ (\nabla^2 L)(\theta) (\theta' - \theta)^T \\
& ~ + o(\|\theta' - \theta\|_2^2).
\end{align*}
In this second order approximation, the $\epsilon$-sharpness at $\theta$ would be
\begin{align*}
\frac{\vertiii{(\nabla^{2} L)(\theta)}_2 \epsilon^2}{2\big(1 + L(\theta)\big)} .
\end{align*}

% \newpage
\section{Properties of Deep Rectified Networks}
\label{sec:deep-relu}

Before moving forward to our results, in this section we first introduce the notation used in
the rest of paper.  Most of our results, for clarity,  will be on the deep
rectified feedforward networks with a linear output layer that we describe
below, though they can easily be extended to other architectures (e.g.
convolutional, etc.).

\begin{defin}
    Given $K$ weight matrices $(\theta_k)_{k \leq K}$ with $n_k =
    \text{dim}\big(\text{vec}(\theta_k)\big)$ and $n = \sum_{k=1}^{K}{n_k}$,
    the output $y$ of a deep rectified feedforward networks with a linear output layer is:
\begin{align*}
    y &= \phi_{rect}\Big(\phi_{rect}\big(\cdots \phi_{rect}(x \cdot \theta_{1}) \cdots \big) \cdot \theta_{K - 1}\Big) \cdot \theta_K,
\end{align*}
where
    \begin{itemize}
        \item $x$ is the input to the model, a high-dimensional vector
        \item $\phi_{rect}$ is the rectified elementwise activation
            function~\citep{jarrett2009best, nair2010rectified, glorot2011deep},
            which is the positive part $(z_i)_i~\mapsto~(\max(z_i, 0))_i$.
        \item \text{vec} reshapes a matrix into a vector.
    \end{itemize}
\end{defin}

Note that in our definition we excluded the bias terms, usually found in any
neural architecture. This is done mainly for convenience, to simplify the rendition
of our arguments. However, the arguments can be extended to the case that includes biases (see Appendix~\ref{app:bias}).
Another choice is that of the linear output layer. Having an output
activation function does not affect our argument either: since the loss is a
function of the output activation, it can be rephrased as a function of linear
pre-activation.

Deep rectifier models have certain properties that allows us
in section~\ref{sec:flat} to arbitrary manipulate the flatness of a minimum.

\begin{figure}[h]
\vspace{15pt}
   \centering \includegraphics[width=.45\textwidth]{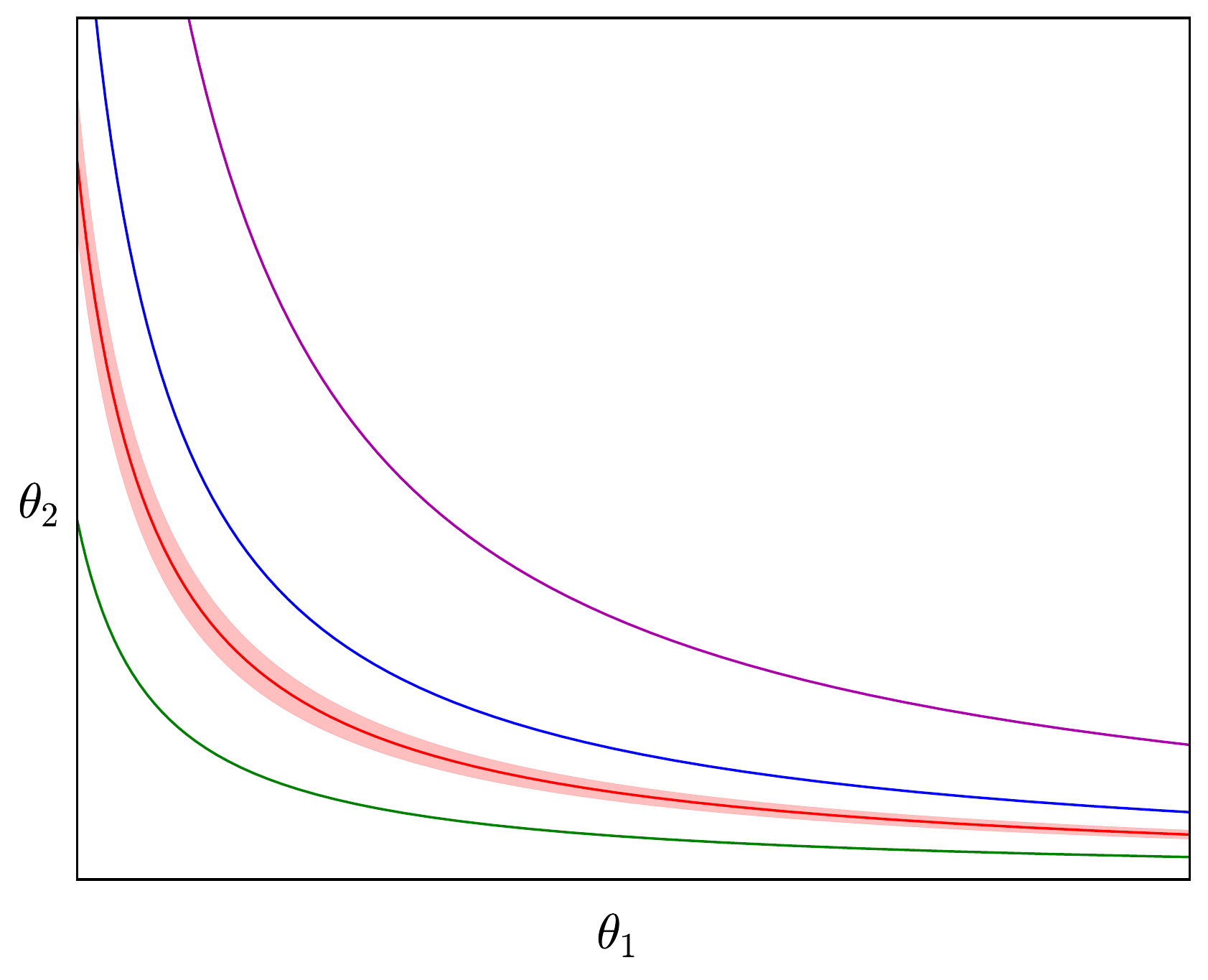}
   \caption{An illustration of the effects of non-negative homogeneity.  The
    graph depicts level curves of the behavior of the loss $L$ embedded into
    the two dimensional parameter space with the axis given by $\theta_1$ and
    $\theta_2$. Specifically, each line of a given color corresponds to the
    parameter assignments $(\theta_1, \theta_2)$ that result observationally in
    the same prediction function $f_{\theta}$. Best seen with colors. }
   \label{fig:hyperbolic}
   \vspace{15pt}
\end{figure}

An important topic for optimization of neural networks is understanding the
non-Euclidean geometry of the parameter space as imposed by the neural
architecture \citep[see, for example][]{amari1998natural}.  In principle, when we take a
step in parameter space what we expect to control is the change in the
behavior of the model (i.e. the mapping of the input $x$ to the output $y$).
In principle we are not interested in the parameters per se, but rather only in
the mapping they represent.

If one defines a measure for the change in the behavior of the model, which
can be done under some assumptions, then, it can be used to define, at any
point in the parameter space, a metric that says what is the equivalent change
in the parameters for a unit of change in the behavior of the model. As it
turns out, for neural networks, this metric is not constant over $\Theta$.
Intuitively, the metric is related to the curvature, and since neural networks
can be highly non-linear, the curvature will not be constant. See
\citet{amari1998natural,PascanuNatural14} for more details. Coming back to the concept
of flatness or sharpness of a minimum, this metric should define the flatness.

However, the geometry of the parameter space is more complicated. Regardless of
the measure chosen to compare two instantiations of a neural network, because of
the structure of the model, it also exhibits a large number of symmetric
configurations that result in exactly the same behavior.  Because the
rectifier activation has the non-negative homogeneity property, as we will see
shortly, one can construct a continuum of points that lead to the same
behavior, hence the metric is singular. Which means that one can exploit these
directions in which the model stays unchanged to shape the neighbourhood around
a minimum in such a way that, by most definitions of flatness, this property can
be controlled. See Figure~\ref{fig:hyperbolic} for a visual depiction, where the
flatness (given here as the distance between the different level curves) can be
changed by moving along the curve.

Let us redefine, for convenience, the \emph{non-negative homogeneity}
property~\citep{neyshabur2015path,bottou2016icml} below. Note that beside this property, the
reason for studying the rectified linear activation is for its widespread
adoption~\citep{krizhevsky2012imagenet, simonyan2014very, szegedy2015going,
he2016deep}.
\begin{defin}
  A given a function $\phi$ is \emph{non-negative homogeneous} if
$$\forall (z, \alpha) \in \R \times \R^{+}, \phi(\alpha z) = \alpha
\phi(z)$$.
\end{defin}

\begin{theorem}
    The rectified function $\phi_{rect}(x) = \max(x,0)$ is non-negative homogeneous.
\end{theorem}

\begin{proof}
    Follows trivially from the constraint that $\alpha>0$, given that $x>0 \Rightarrow \alpha x > 0$, iff $\alpha >0$.

\end{proof}
For a deep rectified neural network it means that:
\begin{align*}
\phi_{rect}\big(x \cdot (\alpha \theta_{1})\big) \cdot \theta_{2} &= \phi_{rect}(x \cdot \theta_{1}) \cdot (\alpha \theta_{2}),
\end{align*}
meaning that for this one (hidden) layer neural network, the parameters
$(\alpha \theta_{1}, \theta_{2})$ is observationally equivalent to
$(\theta_{1}, \alpha \theta_{2})$. This observational equivalence similarly holds for convolutional layers.

Given this non-negative homogeneity, if
$(\theta_{1}, \theta_{2}) \neq (0, 0)$ then $\big\{(\alpha \theta_{1},
\alpha^{-1} \theta_{2}), \alpha > 0 \big\}$ is an infinite set of
observationally equivalent parameters, inducing a strong non-identifiability in
this learning scenario.
Other models like \emph{deep linear networks}~\citep{SaxeMG13},
\emph{leaky rectifiers}~\citep{he2015delving} or
\emph{maxout networks}~\citep{goodfellow2013maxout} also have this non-negative
homogeneity property.

In what follows we will rely on such transformations, in particular we will rely on
the following definition:
\begin{defin}
  For a single hidden layer rectifier feedforward network we define the family of transformations
    $$T_{\alpha}:(\theta_1, \theta_2) \mapsto (\alpha \theta_1, \alpha^{-1} \theta_2)$$
 which we refer to as a $\alpha$-scale transformation.
\end{defin}

Note that a $\alpha$-scale transformation will not affect the generalization, as the
behavior of the function is identical. Also while the transformation is only defined
for a single layer rectified feedforward network, it can trivially be extended to any
architecture having a single rectified network as a submodule, e.g. a deep rectified
feedforward network. For simplicity
and readability we will rely on this definition.

\section{Deep Rectified networks and flat minima}
\label{sec:flat}

In this section we exploit the resulting strong non-identifiability to showcase
a few shortcomings of some definitions of flatness. Although $\alpha$-scale
transformation does not affect the function represented, it allows us to
significantly decrease several measures of flatness. For another definition
of flatness, $\alpha$-scale transformation show that all minima
are equally flat.

\subsection{Volume $\epsilon$-flatness}
\begin{theorem}
For a one-hidden layer rectified neural network of the form
\begin{align*}
y = \phi_{rect}(x \cdot \theta_{1}) \cdot \theta_{2},
\end{align*}
and a minimum $\theta = (\theta_1, \theta_2)$, such that $\theta_1 \neq 0$ and $\theta_2 \neq 0$, $\forall \epsilon > 0$ $C(L, \theta, \epsilon)$ has an infinite volume.
\end{theorem}

We will not consider the solution $\theta$ where any of the weight matrices
$\theta_1, \theta_2$ is zero, $\theta_{1} = 0$ or $\theta_{2} = 0$, as it
results in a constant function which we will assume to give poor training
performance. For $\alpha > 0$, the $\alpha$-scale transformation $T_{\alpha}:(\theta_{1},
\theta_{2}) \mapsto (\alpha \theta_{1}, \alpha^{-1} \theta_{2})$ has Jacobian
determinant $\alpha^{n_1 - n_2}$, where once again $n_1 = dim\big(\text{vec}(\theta_1)\big)$ and $n_2 = dim\big(\text{vec}(\theta_2)\big)$. Note that the Jacobian determinant of this
linear transformation is
the change in the volume induced by
$T_{\alpha}$ and $T_{\alpha} \circ T_{\beta} = T_{\alpha \beta}$.
We show below that there is a
connected region containing $\theta$ with infinite volume and where the error
remains approximately constant.

\begin{proof}
We will first introduce a small region with approximately constant error around $\theta$ with non-zero volume. Given $\epsilon > 0$ and if we consider the loss function continuous with respect to the parameter, $C(L, \theta, \epsilon)$ is an open set containing $\theta$.
Since we also have $\theta_1 \neq 0$ and $\theta_2 \neq 0$, let $r > 0$ such that the $\L_{\infty}$ ball $B_{\infty}(r, \theta)$ is in $C(L, \theta, \epsilon)$ and has empty intersection with $\{\theta', \theta'_1 = 0\}$. Let $v = (2r)^{n_1 + n_2} > 0$ the volume of $B_{\infty}(r, \theta)$.

Since the Jacobian determinant of $T_{\alpha}$ is the multiplicative change of induced by $T_{\alpha}$, the volume of $T_{\alpha}\big(B_{\infty}(r, \theta)\big)$ is $v \alpha^{n_1 - n_2}$. If $n_1 \neq n_2$, we can arbitrarily grow the volume of $T_{\alpha}\big(B_{\infty}(r, \theta)\big)$, with error within an $\epsilon$-interval of $L(\theta)$, by having $\alpha$ tends to $+\infty$ if $n_1 > n_2$ or to $0$ otherwise.

If $n_1 = n_2$, $\forall \alpha' > 0, T_{\alpha'}\Big(B_{\infty}(r, \theta)\Big)$ has volume $v$. Let $C' = \mathop{\bigcup}_{\alpha' > 0} T_{\alpha'}\Big(B_{\infty}(r, \theta)\Big)$. $C'$ is a connected region where the error remains approximately constant, i.e. within an $\epsilon$-interval of $L(\theta)$.

\begin{figure}[h]
\vspace{15pt}
   \centering \includegraphics[width=.45\textwidth]{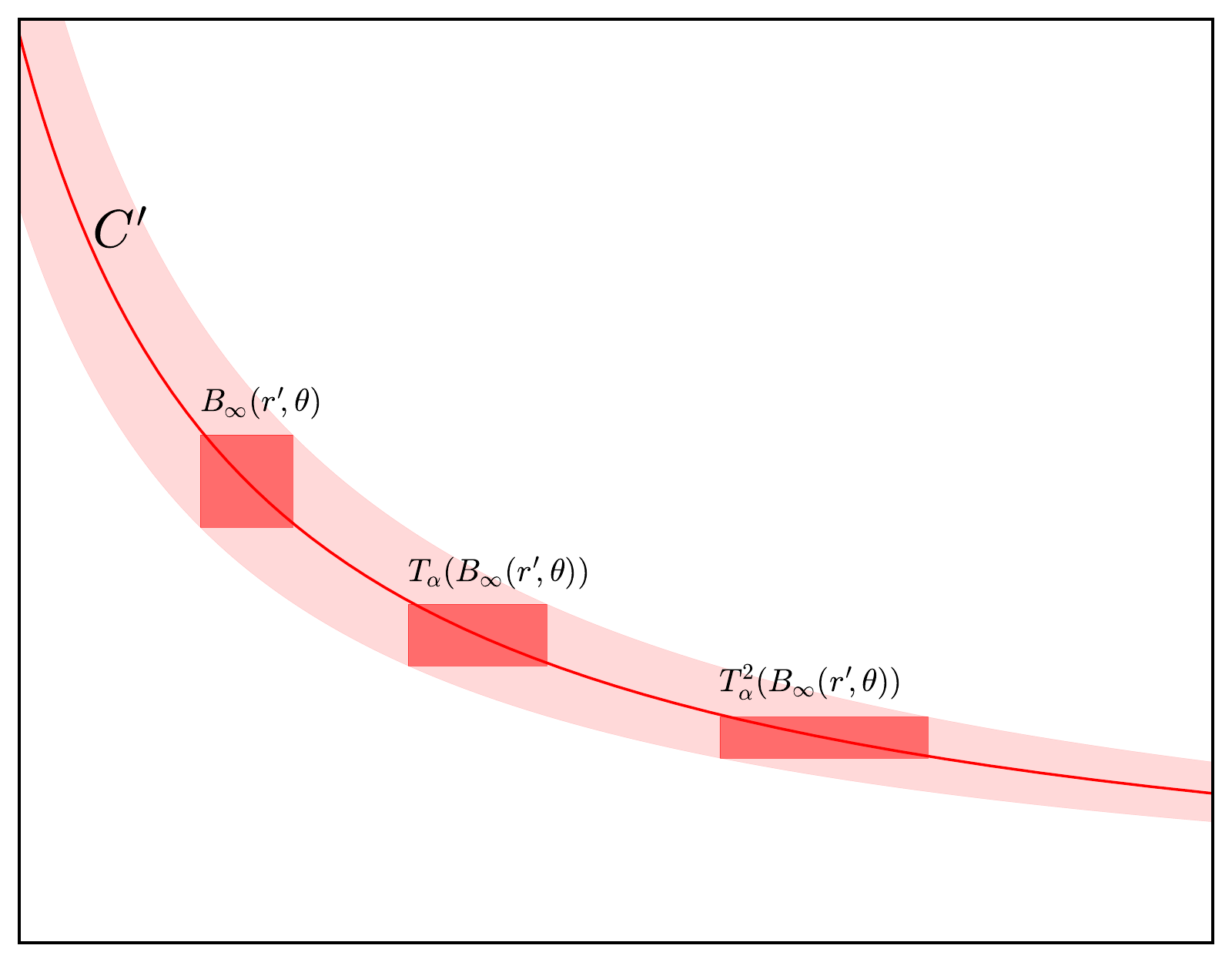}
   \caption{An illustration of how we build different disjoint volumes using $T_{\alpha}$. In this two-dimensional example, $T_{\alpha}\big(B_{\infty}(r', \theta)\big)$ and $B_{\infty}(r', \theta)$ have the same volume. $B_{\infty}(r', \theta), T_{\alpha}\big(B_{\infty}(r', \theta)\big), T_{\alpha}^{2}\big(B_{\infty}(r', \theta)\big), \dots$ will therefore be a sequence of disjoint constant volumes. $C'$ will therefore have an infinite volume. Best seen with colors.
    }
  \vspace{15pt}
   \label{fig:volumes}
\end{figure}

Let $\alpha = 2\frac{\|\theta_1\|_{\infty} + r}{\|\theta_1\|_{\infty} - r}$. Since
\begin{align*}
B_{\infty}(r, \theta) = B_{\infty}(r, \theta_1) \times B_{\infty}(r, \theta_2),
\end{align*}
where $\times$ is the Cartesian set product, we have
\begin{align*}
T_{\alpha}\big(B_{\infty}(r, \theta)\big) = B_{\infty}(\alpha r, \alpha \theta_1) \times B_{\infty}(\alpha^{-1} r, \alpha^{-1} \theta_2).
\end{align*}
Therefore, $T_{\alpha}\big(B_{\infty}(r, \theta)\big) \cap B_{\infty}(r, \theta) = \emptyset$ (see Figure~\ref{fig:volumes}).

Similarly, $B_{\infty}(r, \theta), T_{\alpha}\big(B_{\infty}(r, \theta)\big), T_{\alpha}^{2}\big(B_{\infty}(r, \theta)\big), \dots$ are disjoint and have volume $v$. We have also $T_{\alpha}^{k}\big(B_{\infty}(r', \theta)\big) = T_{\alpha^{k}}\big(B_{\infty}(r', \theta)\big) \in C'$. The volume of $C'$ is then lower bounded by $0 < v + v + v + \cdots$ and is therefore infinite. $C(L, \theta, \epsilon)$ has then infinite volume too, making the volume $\epsilon$-flatness of $\theta$ infinite.
\end{proof}

This theorem can generalize to rectified neural networks in general with a similar proof. Given that every minimum has an infinitely large region (volume-wise) in which the error remains approximately constant, that means that every minimum would be infinitely flat according to the volume $\epsilon$-flatness. Since all minima are equally flat, it is not possible to use volume $\epsilon$-flatness to gauge the generalization property of a minimum.

\subsection{Hessian-based measures}
\label{sec:gradient-hessian}

The non-Euclidean geometry of the parameter space, coupled with the manifolds
of observationally equal behavior of the model, allows one to move from one
region of the parameter space to another, changing the curvature of the model
without actually changing the function. This approach has been used with
success to improve optimization, by moving from a region of high curvature to a
region of well behaved curvature~\citep[e.g.][]{DesjardinsNatural15,
salimans2016weight}.  In this section we look at two widely used measures of
the Hessian, the spectral radius and trace, showing that either of these values
can be manipulated without actually changing the behavior of the function.
If the flatness of a minimum is defined by any of these quantities, then it could
also be easily manipulated.
\begin{theorem}
The gradient and Hessian of the loss $L$ with respect to $\theta$ can be modified
by $T_{\alpha}$.
\end{theorem}
\begin{proof}
\begin{align*}
L(\theta_1, \theta_2) = L(\alpha \theta_{1}, \alpha^{-1} \theta_{2}),
\end{align*}
we have then by differentiation
\begin{align*}
(\nabla L)(\theta_1, \theta_2) = (\nabla L)(\alpha \theta_{1}, \alpha^{-1} \theta_{2}) \left[\begin{array}{cc}
\alpha \Id_{n_1} & 0 \\
0 & \alpha^{-1} \Id_{n_2}
\end{array} \right] \\
\Leftrightarrow (\nabla L)(\alpha \theta_{1}, \alpha^{-1} \theta_{2}) = (\nabla L)(\theta_1, \theta_2) \left[\begin{array}{cc}
\alpha^{-1} \Id_{n_1} & 0 \\
0 & \alpha \Id_{n_2}
\end{array} \right]
\end{align*}
and
\begin{align*}
&(\nabla^{2} L)(\alpha \theta_{1}, \alpha^{-1} \theta_{2}) \\
&= \left[\begin{array}{cc}
\alpha^{-1} \Id_{n_1} & 0 \\
0 & \alpha \Id_{n_2}
\end{array} \right] (\nabla^{2} L)(\theta_1, \theta_2) \left[\begin{array}{cc}
\alpha^{-1} \Id_{n_1} & 0 \\
0 & \alpha \Id_{n_2}
\end{array} \right].
\end{align*}
\end{proof}
\paragraph{Sharpest direction} Through these transformations we can easily
find, for any critical point which is a minimum with non-zero Hessian, an
observationally equivalent parameter whose Hessian has an arbitrarily large
spectral norm.

\begin{theorem}
For a one-hidden layer rectified neural network of the form
\begin{align*}
y = \phi_{rect}(x \cdot \theta_{1}) \cdot \theta_{2},
\end{align*}
and critical point $\theta = (\theta_1, \theta_2)$ being a minimum for $L$, such that $(\nabla^{2} L)(\theta) \neq 0$, $\forall M > 0, \exists \alpha > 0, \vertiii{(\nabla^{2} L)\big(T_{\alpha}(\theta)\big)}_2 \geq M$ where $\vertiii{(\nabla^{2} L)\big(T_{\alpha}(\theta)\big)}_2$ is the spectral norm of $(\nabla^{2} L)\big(T_{\alpha}(\theta)\big)$.
\end{theorem}

\begin{proof}
The trace of a symmetric matrix is the sum of its eigenvalues and a real symmetric matrix can be diagonalized in $\R$, therefore if the Hessian is non-zero, there is one non-zero positive diagonal element. Without loss of generality, we will assume that this non-zero element of value $\gamma > 0$ corresponds to an element in $\theta_1$. Therefore the Frobenius norm $\vertiii{(\nabla^{2} L)\big(T_{\alpha}(\theta)\big)}_{F}$ of
\begin{align*}
&(\nabla^{2} L)(\alpha \theta_{1}, \alpha^{-1} \theta_{2}) \\
&= \left[\begin{array}{cc}
\alpha^{-1} \Id_{n_1} & 0 \\
0 & \alpha \Id_{n_2}
\end{array} \right] (\nabla^{2} L)(\theta_1, \theta_2) \left[\begin{array}{cc}
\alpha^{-1} \Id_{n_1} & 0 \\
0 & \alpha \Id_{n_2}
\end{array} \right].
\end{align*}
is lower bounded by $\alpha^{-2} \gamma$.

Since all norms are equivalent in finite dimension, there exists a constant $r > 0$ such that $r \vertiii{A}_{F} \leq \vertiii{A}_{2}$ for all symmetric matrices $A$. So by picking $\alpha < \sqrt{\frac{r \gamma}{M}}$, we are guaranteed that $\vertiii{(\nabla^{2} L)\big(T_{\alpha}(\theta)\big)}_2 \geq M$.
\end{proof}

Any minimum with non-zero Hessian will be observationally equivalent to a
minimum whose Hessian has an arbitrarily large spectral norm. Therefore for any
minimum in the loss function, if there exists another minimum that generalizes
better then there exists another minimum that generalizes better and is also
sharper according the spectral norm of the Hessian. The spectral norm of
critical points' Hessian becomes as a result less relevant as a measure of
potential generalization error. Moreover, since the spectral norm lower bounds
the trace for a positive semi-definite symmetric matrix, the same conclusion
can be drawn for the trace.

\paragraph{Many directions} However, some notion of sharpness might take into
account the entire eigenspectrum of the Hessian as opposed to its largest
eigenvalue, for instance, \citet{chaudhari2016entropy} describe the notion of
\emph{wide valleys}, allowing the presence of very few large eigenvalues. We
can generalize the transformations between observationally equivalent
parameters to deeper neural networks with $K - 1$ hidden layers: for $\alpha_k
> 0, T_{\alpha}: (\theta_k)_{k \leq K} \mapsto (\alpha_k \theta_k)_{k \in K}$
with $\prod_{k=1}^{K}{\alpha_k} = 1$. If we define
\begin{align*}
D_{\alpha} = \left[\begin{array}{cccc}
\alpha_{1}^{-1} \Id_{n_1} & 0 & \cdots & 0 \\
0 & \alpha_{2}^{-1} \Id_{n_2} & \cdots & 0 \\
\vdots & \vdots & \ddots & \vdots \\
0 & 0 & \cdots & \alpha_{K}^{-1} \Id_{n_K}
\end{array} \right]
\end{align*}
then the first and second derivatives at $T_{\alpha}(\theta)$ will be
\begin{align*}
(\nabla L)\big(T_{\alpha}(\theta)\big) =& (\nabla L)(\theta) D_{\alpha} \\
(\nabla^2 L)\big(T_{\alpha}(\theta)\big) =& D_{\alpha} (\nabla^2 L)(\theta) D_{\alpha}.
\end{align*}
We will show to which extent you can increase several eigenvalues of $(\nabla^2 L)\big(T_{\alpha}(\theta)\big)$ by varying $\alpha$.

\begin{defin}
For each $n \times n$ matrix $A$, we define the vector $\lambda(A)$ of sorted singular values of $A$ with their multiplicity $\lambda_1(A) \geq \lambda_2(A) \geq \cdots \geq \lambda_n(A)$.
\end{defin}
If $A$ is symmetric positive semi-definite, $\lambda(A)$ is also the vector of its sorted eigenvalues.

\begin{theorem}
For a $(K - 1)$-hidden layer rectified neural network of the form
\begin{align*}
y &= \phi_{rect}(\phi_{rect}(\cdots \phi_{rect}(x \cdot \theta_{1}) \cdots ) \cdot \theta_{K - 1}) \cdot \theta_K,
\end{align*}
and critical point $\theta = (\theta_k)_{k \leq K}$ being a minimum for $L$, such that $(\nabla^{2} L)(\theta)$ has rank $r = \text{rank}\big((\nabla^2 L)(\theta)\big)$, $\forall M > 0, \exists \alpha > 0$ such that $\Big(r - \min_{k \leq K}(n_k)\Big)$ eigenvalues are greater than $M$.
\end{theorem}

\begin{proof}
For simplicity, we will note $\sqrt{M}$ the \emph{principal square root} of a symmetric positive-semidefinite matrix $M$. The eigenvalues of $\sqrt{M}$ are the square root of the eigenvalues of $M$ and are its \emph{singular values}. By definition, the \emph{singular values} of $\sqrt{(\nabla^2 L)(\theta)} D_{\alpha}$ are the square root of the eigenvalues of $D_{\alpha} (\nabla^2 L)(\theta) D_{\alpha}$. %By matrix equivalence, we know that the eigenvalue spectrum of $D_{\alpha} (\nabla^2 L)(\theta) D_{\alpha}$ is the same as that of $(\nabla^2 L)(\theta) D_{\alpha}^{2}$.
Without loss of generality, we consider $\min_{k \leq K}(n_k) = n_K$ and choose $\forall k < K, \alpha_k = \beta^{-1}$ and $\alpha_K = \beta^{K-1}$. Since $D_{\alpha}$ and $\sqrt{(\nabla^2 L)(\theta)}$ are positive symmetric semi-definite matrices, we can apply the multiplicative Horn inequalities~\citep{klyachko2000random} on singular values of the product $\sqrt{(\nabla^2 L)(\theta)} D_{\alpha}$:
\begin{align*}
\forall i \leq n, &j \leq (n - n_K), \\
&\lambda_{i + j - n}\big((\nabla^2 L)(\theta) D_{\alpha}^{2}\big) \geq \lambda_{i}\big((\nabla^2 L)(\theta)\big) \beta^{2}.
\end{align*}
By choosing $\beta > \sqrt{\frac{M}{\lambda_{r}\big((\nabla^2 L)(\theta)\big)}}$, since we have $\forall i~\leq~r, \lambda_{i}\big((\nabla^2 L)(\theta)\big) \geq \lambda_{r}\big((\nabla^2 L)(\theta)\big) > 0$ we can conclude that
\begin{align*}
\forall i \leq (r - n_K),& \\
\lambda_{i}\big((\nabla^2 L)(\theta) D_{\alpha}^{2}\big) &\geq \lambda_{i + n_k}\big((\nabla^2 L)(\theta)\big) \beta^{2} \\
&\geq \lambda_{r}\big((\nabla^2 L)(\theta)\big) \beta^{2} > M.
\end{align*}
\end{proof}

It means that there exists an observationally equivalent parameter with at
least $\Big(r - \min_{k \leq K}(n_k) \Big)$
arbitrarily large eigenvalues. Since~\citet{sagun2016singularity} seems to
suggests that rank deficiency in the Hessian is due to over-parametrization
of the model, one could conjecture that $\Big(r - \min_{k \leq K}(n_k)\Big)$ can be high for thin and deep neural networks, resulting in a majority of large eigenvalues. Therefore, it would still be possible to obtain an equivalent parameter with large Hessian eigenvalues, i.e. sharp in multiple directions.

\subsection{$\epsilon$-sharpness}
\begin{figure}[h]
\vspace{15pt}
   \centering \includegraphics[width=.45\textwidth]{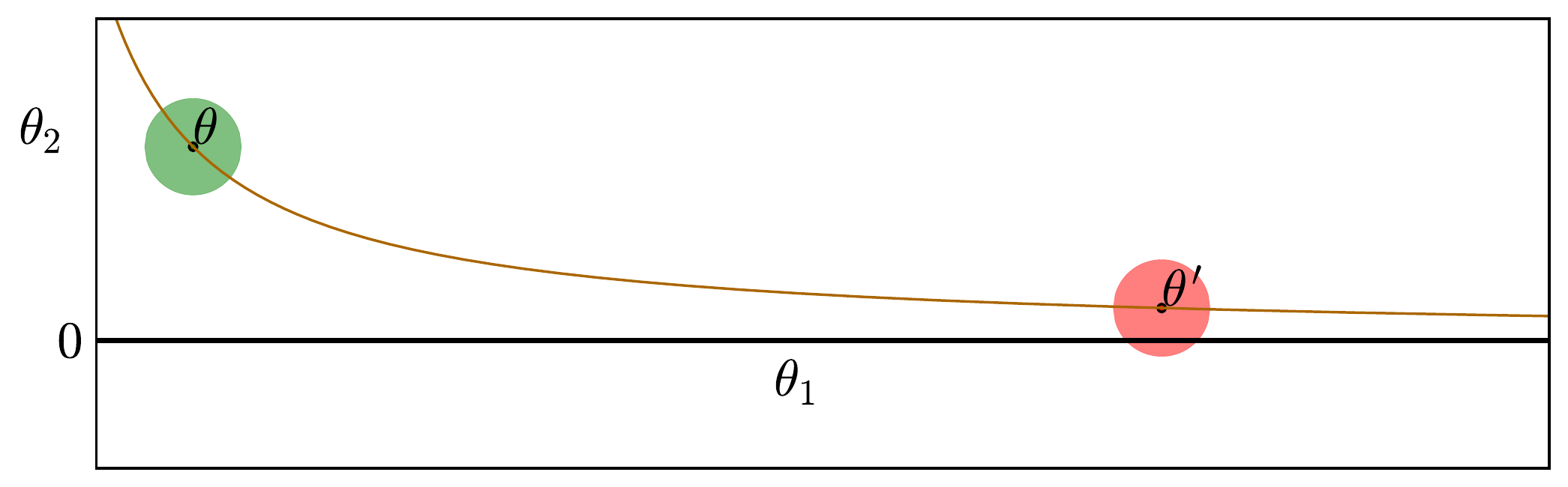}
   \caption{An illustration of how we exploit non-identifiability and its
            particular geometry to obtain sharper
            minima: although $\theta$ is far from the $\theta_2 = 0$ line,
            the observationally equivalent parameter $\theta'$ is closer.
            The green and red circle centered on each of these points have the
            same radius.
            Best seen with colors. }
   \label{fig:epsilon_sharpness}
   \vspace{15pt}
\end{figure}
\label{sec:sharpness}
We have redefined for $\epsilon > 0$ the $\epsilon$-sharpness of~\citet{keskar2016large} as follow
\begin{align*}
\frac{\max_{\theta' \in B_2(\epsilon, \theta)}\big(L(\theta') - L(\theta)\big)}{1 + L(\theta)}
\end{align*}
where $B_2(\epsilon, \theta)$ is the Euclidean ball of radius $\epsilon$ centered on $\theta$.
This modification will demonstrate more clearly the issues of that metric as a
measure of probable generalization. If we use $K = 2$ and $(\theta_1,
\theta_2)$ corresponding to a non-constant function, i.e. $\theta_1 \neq 0$ and
$\theta_2 \neq 0$, then we can define $\alpha =
\frac{\epsilon}{\|\theta_1\|_2}$. We will now consider the observationally
equivalent parameter $T_{\alpha}(\theta_1, \theta_2) = (\epsilon
\frac{\theta_1}{\|\theta_1\|_2}, \alpha^{-1} \theta_2)$. Given that
$\|\theta_1\|_2 \leq \|\theta\|_2$, we have that $(0, \alpha^{-1} \theta_2) \in
B_2\big(\epsilon, T_{\alpha}(\theta)\big)$, making the maximum loss in this neighborhood at least
as high as the best constant-valued function, incurring relatively high
sharpness. Figure~\ref{fig:epsilon_sharpness} provides a visualization of the proof.

For rectified neural network every minimum is observationally equivalent to a minimum that generalizes as well but with high $\epsilon$-sharpness. This also applies when using the \emph{full-space} $\epsilon$-sharpness used by~\citet{keskar2016large}. We can prove this similarly using the equivalence of norms in finite dimensional vector spaces and the fact that for $c > 0, \epsilon > 0, \epsilon \leq \epsilon (c + 1)$ (see~\citet{keskar2016large}). We have not been able to show a similar problem with \emph{random subspace} $\epsilon$-sharpness used by~\citet{keskar2016large}, i.e. a restriction of the maximization to a random subspace, which could relate to the notion of \emph{wide valleys} described by~\citet{chaudhari2016entropy}. \\

By exploiting the non-Euclidean geometry and non-identifiability of rectified neural
networks, we were able to demonstrate some of the limits of using typical
definitions of minimum's flatness as core explanation for generalization.

\section{Allowing reparametrizations}
\label{sec:reparametrization}

In the previous section~\ref{sec:flat} we explored the case of a fixed parametrization,
that of deep rectifier models. In this section we demonstrate a simple observation.
If we are allowed to change the parametrization of some function $f$, we can obtain
arbitrarily different geometries without affecting how the function evaluates on
unseen data. The same holds for reparametrization of the input space.
The implication is that the correlation between the geometry of the parameter space (and
hence the error surface) and the behavior of a given function is meaningless if not
preconditioned on the specific parametrization of the model.

\subsection{Model reparametrization}
\label{sec:model-reparam}

One thing that needs to be considered when relating flatness of minima to their
probable generalization is that the choice of parametrization and its
associated geometry are arbitrary. Since we are interested in finding a
prediction function in a given family of functions, no reparametrization of
this family should influence generalization of any of these functions. Given a
bijection $g$ onto $\theta$, we can define new transformed parameter $\eta =
g^{-1}(\theta)$. Since $\theta$ and $\eta$ represent in different space the
same prediction function, they should generalize as well.

\begin{figure}[h]
\vspace{15pt}
   \centering \subfigure[Loss function with default parametrization]{
     \includegraphics[width=.45\textwidth]{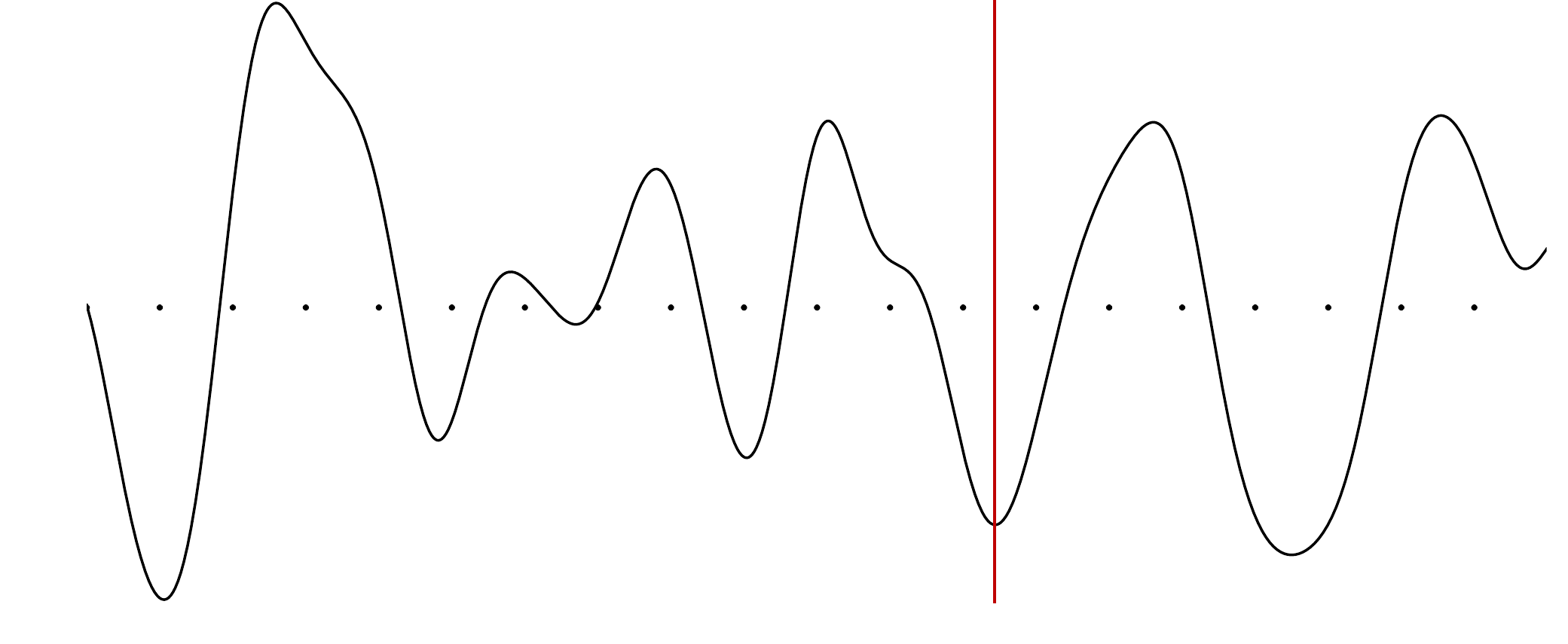}}
   \subfigure[Loss function with reparametrization]{
     \includegraphics[width=.45\textwidth]{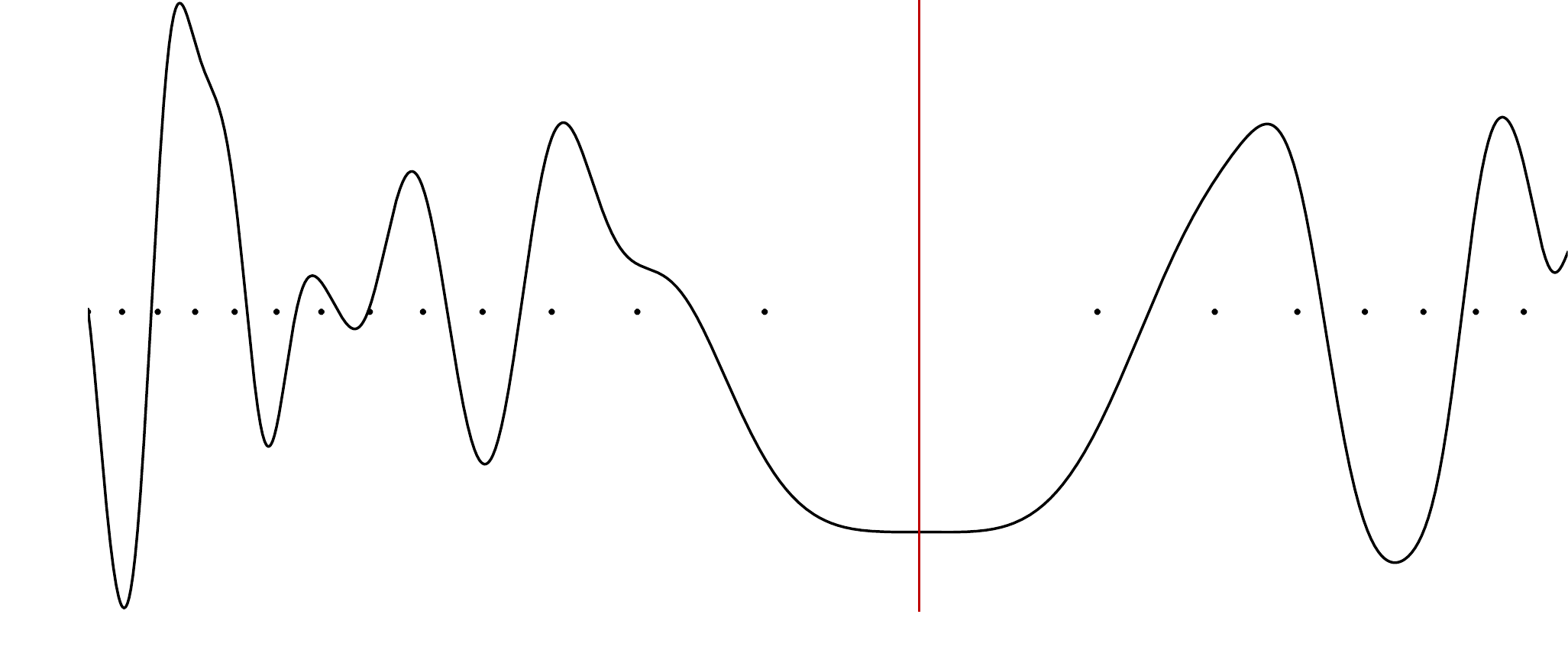}}
   \subfigure[Loss function with another reparametrization]{
     \includegraphics[width=.45\textwidth]{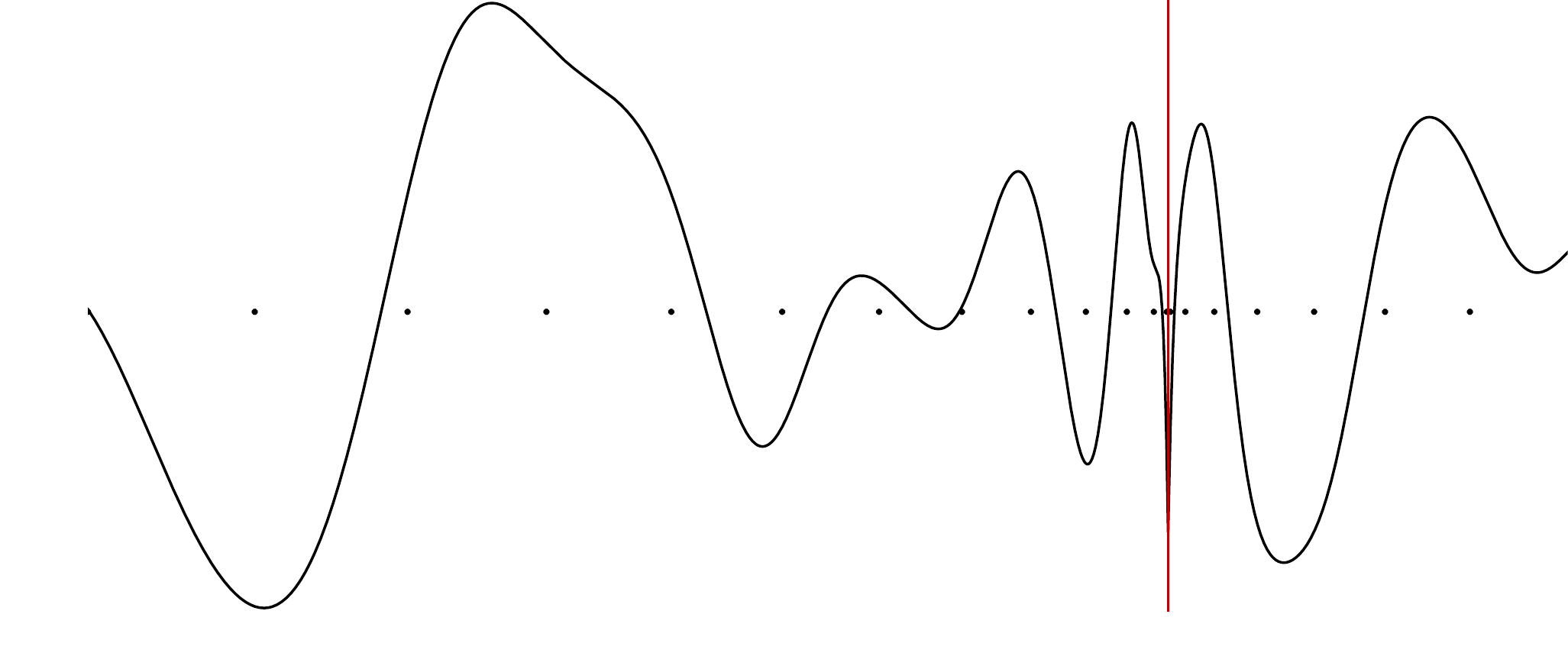}}
   \caption{A one-dimensional example on how much the geometry of the loss
    function depends on the parameter space chosen. The $x$-axis is the
    parameter value and the $y$-axis is the loss. The points correspond to a
    regular grid in the default parametrization. In the default
    parametrization, all minima have roughly the same curvature but with a
    careful choice of reparametrization, it is possible to turn a minimum
    significantly flatter or sharper than the others. Reparametrizations in
    this figure are of the form $\eta = (|\theta - \hat{\theta}|^2 +
    b)^{a}(\theta - \hat{\theta})$ where $b \geq 0, a > -\frac{1}{2}$ and $\hat{\theta}$
    is shown with the red vertical line.
    }
   \label{fig:stretch}
   \vspace{15pt}
\end{figure}

Let's call $L_{\eta} = L \circ g$ the loss function with respect to the new
parameter $\eta$. We generalize the derivation of
Subsection~\ref{sec:gradient-hessian}:
\begin{align*}
L_{\eta}(\eta) &= L\big(g(\eta)\big) \\
\Rightarrow (\nabla L_{\eta})(\eta) &= (\nabla L)\big(g(\eta)\big) (\nabla g)(\eta) \\
\Rightarrow (\nabla^2 L_{\eta})(\eta) &= (\nabla g)(\eta)^T (\nabla^2 L)\big(g(\eta)\big) (\nabla g)(\eta) \\
& ~~~~~~ + (\nabla L)\big(g(\eta)\big) (\nabla^2 g)(\eta).
\end{align*}
At a differentiable critical point, we have by definition $(\nabla L)\big(g(\eta)\big) = 0$, therefore the transformed Hessian at a critical point becomes
\begin{align*}
(\nabla^2 L_{\eta})(\eta) = (\nabla g)(\eta)^T (\nabla^2 L)\big(g(\eta)\big) (\nabla g)(\eta).
\end{align*}This means that by reparametrizing the problem we can modify to a large extent the geometry of the loss function so as to have sharp minima of $L$ in $\theta$ correspond to flat minima of $L_{\eta}$ in $\eta = g^{-1}(\theta)$ and conversely. Figure~\ref{fig:stretch} illustrates that point in one dimension. Several practical~\citep{dinh2014nice,rezende2015variational,kingma2016improving,dinh2016density} and theoretical works~\citep{hyvarinen1999nonlinear} show how powerful bijections can be. We can also note that the formula for the transformed Hessian at a critical point also applies if $g$ is not invertible, $g$ would just need to be surjective over $\Theta$ in order to cover exactly the same family of prediction functions
\begin{align*}
\{f_{\theta}, \theta \in \Theta \} = \{ f_{g(\eta)}, \eta \in g^{-1}(\Theta)\} .
\end{align*}
We show in Appendix~\ref{appendix:radial}, bijections that allow us to perturb the relative flatness between a finite number of minima.

Instances of commonly used reparametrization are \emph{batch
normalization}~\citep{ioffe2015batch}, or the \emph{virtual batch
normalization} variant~\cite{salimans2016improved}, and \emph{weight
normalization}~\citep{badrinarayanan2015understanding,salimans2016weight,arpit2016normalization}. \citet{im2016empirical} have plotted
how the loss function landscape was affected by batch normalization. However,
we will focus on weight normalization reparametrization as the
analysis will be simpler, but the intuition with batch normalization will be
similar. Weight normalization reparametrizes a nonzero weight $w$ as $w = s
\frac{v}{\|v\|_2}$ with the new parameter being the scale $s$ and the
unnormalized weight $v \neq 0$.

Since we can observe that $w$ is invariant to scaling of $v$, reasoning
similar to Section~\ref{sec:deep-relu} can be applied with the simpler
transformations $T'_{\alpha}: v \mapsto \alpha v$ for $\alpha \neq 0$. Moreover,
since this transformation is a simpler isotropic scaling, the conclusion that
we can draw can be actually more powerful with respect to $v$:
\setlist{nolistsep}
\begin{itemize}
\setlength\itemsep{0.5em}
\item every minimum has infinite volume $\epsilon$-sharpness;
\item every minimum is observationally equivalent to an infinitely sharp minimum and to an infinitely flat minimum when considering nonzero eigenvalues of the Hessian;
\item every minimum is observationally equivalent to a minimum with arbitrarily low full-space and random subspace $\epsilon$-sharpness and a minimum with high full-space $\epsilon$-sharpness.
\end{itemize}
This further weakens the link between the flatness of a minimum and the generalization property of the associated prediction function when a specific parameter space has not been specified and explained beforehand.

\subsection{Input representation}
\label{sec:sensitivity}
As we conclude that the notion of flatness for a minimum in the loss function by itself is not sufficient to determine its generalization ability in the general case, we can choose to focus instead on properties of the prediction function instead. Motivated by some work in \emph{adversarial examples}~\citep{szegedy2013intriguing,goodfellow2014explaining} for deep neural networks, one could decide on its generalization property by analyzing the gradient of the prediction function on examples. Intuitively, if the gradient is small on typical points from the distribution or has a small Lipschitz constant, then a small change in the input should not incur a large change in the prediction.

But this infinitesimal reasoning is once again very dependent of the local geometry of the input space. For an invertible preprocessing $\xi^{-1}$, e.g. \emph{feature standardization}, \emph{whitening} or \emph{gaussianization}~\citep{NIPS2000_1856}, we will call $f_{\xi} = f \circ \xi$ the prediction function on the preprocessed input $u = \xi^{-1}(x)$. We can reproduce the derivation in Section~\ref{sec:reparametrization} to obtain
\begin{align*}
\frac{\partial f_{\xi}}{\partial u^{T}}\big(\xi(u)\big) = \frac{\partial f}{\partial x^{T}}\big(\xi(u)\big) \frac{\partial \xi}{\partial u^{T}}(u).
\end{align*}
As we can alter significantly the relative magnitude of the gradient at each point, analyzing the amplitude of the gradient of the prediction function might prove problematic if the choice of the input space have not been explained beforehand. This remark applies in applications involving images, sound or other signals with invariances~\citep{LarsenSW15}. For example, \citet{theis2015note} show for images how a small drift of one to four pixels can incur a large difference in terms of $\L_2$ norm.

\section{Discussion}

It has been observed empirically that minima found by standard deep learning
algorithms that generalize well tend to be flatter than found minima that did
not generalize well~\citep{chaudhari2016entropy,keskar2016large}. However, when
following several definitions of flatness, we have shown that the conclusion
that flat minima should generalize better than sharp ones cannot be applied as
is without further context. Previously used definitions fail to account
for the complex geometry of some commonly used deep
architectures. In particular, the non-identifiability of the model induced
by symmetries, allows one to alter the flatness of a minimum without affecting
the function it represents. Additionally the whole geometry of the error surface
with respect to the parameters can be changed arbitrarily under different
parametrizations.
In the spirit of~\cite{SwirszczLocalMinima16}, our work indicates that more care is needed
to define flatness to avoid degeneracies of the geometry of the model under
study. Also such a concept
can not be divorced from the particular parametrization of the model or
input space.

\ifdefined\isaccepted
\section*{Acknowledgements}
The authors would like to thank Grzegorz \'{S}wirszcz for an insightful discussion
of the paper, Harm De Vries, Yann Dauphin, Jascha
Sohl-Dickstein and César Laurent for useful discussions about optimization,
Danilo Rezende for explaining universal approximation using normalizing flows
and Kyle Kastner, Adriana Romero, Junyoung Chung, Nicolas Ballas, Aaron
Courville, George Dahl, Yaroslav Ganin, Prajit Ramachandran, Çağlar Gülçehre,
Ahmed Touati and the ICML reviewers for useful feedback.
\fi

\bibliography{local}
\bibliographystyle{icml2017}

\ifdefined\isaccepted
\appendix

\section{Radial transformations}
\label{appendix:radial}
We show an elementary transformation to locally perturb the geometry of a finite-dimensional vector space and therefore affect the relative flatness between a finite number minima, at least in terms of spectral norm of the Hessian.
We define the function:
\begin{align*}
\forall \delta > 0, \forall \rho \in ]0, \delta[, &\forall (r, \hat{r}) \in \R_{+} \times ]0, \delta[, \\
\psi(r, \hat{r}, \delta, \rho) = ~ &\mathbbm{1}\big(r \notin [0, \delta]\big) ~ r + \mathbbm{1}\big(r \in [0, \hat{r}]\big) ~ \rho ~ \frac{r}{\hat{r}} \\
&+ \mathbbm{1}\big(r \in ]\hat{r}, \delta]\big) ~ \Big((\rho - \delta) ~ \frac{r - \delta}{\hat{r} - \delta} + \delta\Big)\\
\psi'(r, \hat{r}, \delta, \rho) = ~ &\mathbbm{1}\big(r \notin [0, \delta]\big) + \mathbbm{1}\big(r \in [0, \hat{r}]\big) ~ \frac{\rho}{\hat{r}} \\
&+ \mathbbm{1}\big(r \in ]\hat{r}, \delta]\big) ~ \frac{\rho - \delta}{\hat{r} - \delta}
\end{align*}

\begin{figure}[t]
\vspace{15pt}
  \centering
   \subfigure[$\psi(r, \hat{r}, \delta, \rho)$]{
    \includegraphics[width=.3\textwidth]{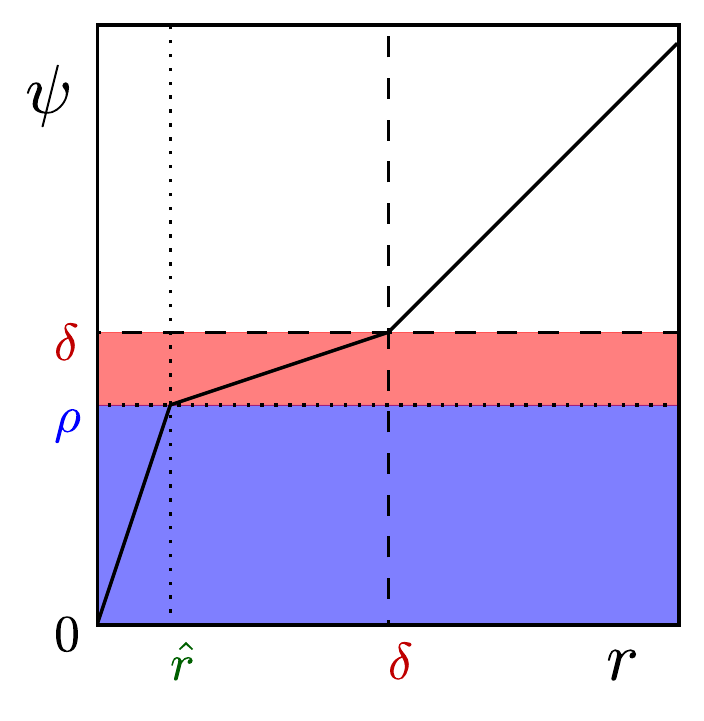}}
   \subfigure[$g^{-1}(\theta)$]{
    \includegraphics[width=.3\textwidth]{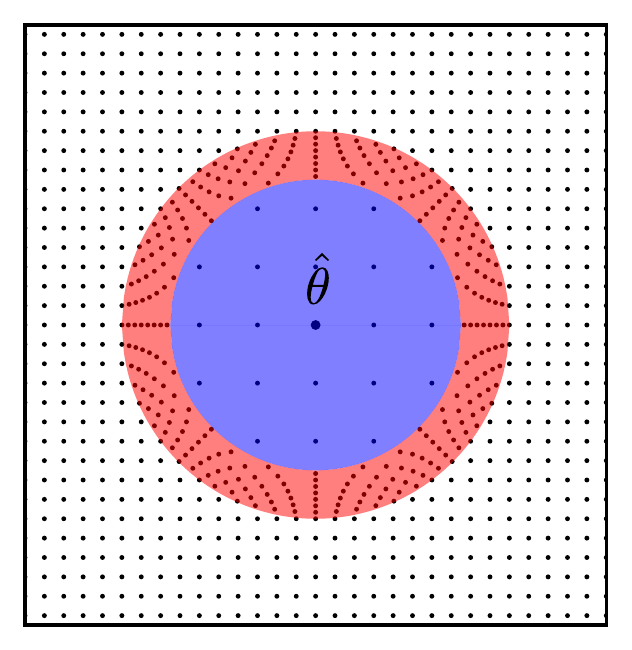}}
   \caption{An example of a radial transformation on a $2$-dimensional space. We can see that only the area in blue and red, i.e. inside $B_2(\hat{\theta}, \delta)$, are affected. Best seen with colors.}
   \label{fig:radial}
   \vspace{15pt}
\end{figure}

For a parameter $\hat{\theta} \in \Theta$ and $\delta > 0, \rho \in ]0, \delta[, \hat{r} \in ]0, \delta[$, inspired by the \emph{radial flows}~\citep{rezende2015variational} in we can define the \emph{radial transformations}
\begin{align*}
\forall \theta \in \Theta, ~ g^{-1}(\theta) =  \frac{\psi\Big(\|\theta - \hat{\theta}\|_2, \hat{r}, \delta, \rho\Big)}{\|\theta - \hat{\theta}\|_2} ~ \big(\theta - \hat{\theta}\big) + \hat{\theta}
\end{align*}
with Jacobian
\begin{align*}
\forall \theta \in \Theta, ~& (\nabla g^{-1})(\theta) =~ \psi'(r, \hat{r}, \delta, \rho) ~ \Id_{n} \\
&- \mathbbm{1}\big(r \in ]\hat{r}, \delta]\big) \frac{\delta (\hat{r} - \rho)}{r^3 (\hat{r} - \delta)} ~ (\theta - \hat{\theta})^T (\theta - \hat{\theta}) \\
&+ \mathbbm{1}\big(r \in ]\hat{r}, \delta]\big) \frac{\delta (\hat{r} - \rho)}{r (\hat{r} - \delta)} ~ \Id_{n},
\end{align*}
with $r = \|\theta - \hat{\theta}\|_2$.

First, we can observe in Figure~\ref{fig:radial} that these transformations are purely local: they only have an effect inside the ball $B_2(\hat{\theta}, \delta)$. Through these transformations, you can arbitrarily perturb the ranking between several minima in terms of flatness as described in Subsection~\ref{sec:model-reparam}.

\newpage
\section{Considering the bias parameter}
\label{app:bias}
When we consider the bias parameter for a one (hidden) layer neural network, the non-negative homogeneity property translates into
\begin{align*}
y &= \phi_{rect}(x \cdot \theta_{1} + b_1) \cdot \theta_{2} + b_{2} \\
&= \phi_{rect}(x \cdot \alpha \theta_{1} + \alpha b_1) \cdot \alpha^{-1} \theta_{2} + b_{2},
\end{align*}
which results in conclusions similar to section~\ref{sec:flat}.

For a deeper rectified neural network, this property results in
\begin{align*}
y &= \phi_{rect}\Big(\phi_{rect}\big(\cdots \phi_{rect}(x \cdot \theta_{1} + b_1) \cdots \big) \cdot \theta_{K - 1} + b_{K-1}\Big) \\
&~~~~ \cdot \theta_K + b_K \\
&= \phi_{rect}\Big(\phi_{rect}\big(\cdots \phi_{rect}(x \cdot \alpha_1 \theta_{1} + \alpha_1 b_1) \cdots \big) \\
&~~~~ \cdot \alpha_{K - 1} \theta_{K - 1} + \prod_{k=1}^{K-1}{\alpha_k} b_{K-1}\Big) \cdot \alpha_{K} \theta_K + b_K
\end{align*}
for $\prod_{k=1}^{K}{\alpha_k} = 1$. This can decrease the amount of eigenvalues of the Hessian that can be arbitrarily influenced.

\section{Rectified neural network and Lipschitz continuity}
Relative to recent works~\citep{hardt2015train, gonen2017fast} assuming
\emph{Lipschitz continuity} of the loss function to derive uniform stability bound,
we make the following observation:
\begin{theorem}
For a one-hidden layer rectified neural network of the form
\begin{align*}
y = \phi_{rect}(x \cdot \theta_{1}) \cdot \theta_{2},
\end{align*}
if $L$ is not constant, then it is not Lipschitz continuous.
\end{theorem}

\begin{proof}
Since a Lipschitz function is necessarily absolutely continuous, we will consider the cases where $L$ is absolutely continuous. First, if $L$ has zero gradient almost everywhere, then $L$ is constant.

Now, if there is a point $\theta$ with non-zero gradient, then by writing
\begin{align*}
  (\nabla L)(\theta_1, \theta_2) = [&(\nabla_{\theta_1} L)(\theta_1, \theta_2)\\
  &(\nabla_{\theta_2} L)(\theta_1, \theta_2)],
\end{align*}
we have
\begin{align*}
  (\nabla L)(\alpha \theta_{1}, \alpha^{-1} \theta_{2}) = [\alpha^{-1} &(\nabla_{\theta_1} L)(\theta_1, \theta_2) \\
  \alpha &(\nabla_{\theta_2} L)(\theta_1, \theta_2)].
\end{align*}
Without loss of generality, we consider $(\nabla_{\theta_1} L)(\theta_1, \theta_2) \neq 0$.
Then the limit of the norm
\begin{align*}
  \| (\nabla L)(\alpha \theta_{1}, \alpha^{-1} \theta_{2}) \|_{2}^{2} = ~&\alpha^{-2} \|(\nabla_{\theta_1} L)(\theta_1, \theta_2)\|_{2}^{2} \\
  & + \alpha^{2} \|(\nabla_{\theta_2} L)(\theta_1, \theta_2)\|_{2}^{2}
\end{align*}
of the gradient goes to $+\infty$ as $\alpha$ goes to $0$.
Therefore, $L$ is not Lipschitz continuous.
\end{proof}
This result can be generalized to several other models containing a one-hidden
layer rectified neural network, including deeper rectified networks.

\newpage
\section{Euclidean distance and input representation}
A natural consequence of Subsection~\ref{sec:sensitivity} is that metrics relying on Euclidean metric like \emph{mean square error} or \emph{Earth-mover distance} will rank very differently models depending on the input representation chosen. Therefore, the choice of input representation is critical when ranking different models based on these metrics. Indeed, bijective transformations as simple as \emph{feature standardization} or \emph{whitening} can change the metric significantly.

On the contrary, ranking resulting from metrics like \emph{f-divergence} and \emph{log-likelihood} are not perturbed by bijective transformations because of the \emph{change of variables formula}.
\fi
\end{document}